\documentclass[letterpaper, 10 pt, conference]{ieeeconf}  %

\IEEEoverridecommandlockouts                              %

\let\labelindent\relax
\usepackage[loadonly]{enumitem}

\usepackage{algorithmic}
\usepackage{algorithm}

\usepackage{amsmath,amsfonts,bm}

\def\eqref#1{equation~\ref{#1}}

\def\1{\bm{1}}

\DeclareMathAlphabet{\mathsfit}{\encodingdefault}{\sfdefault}{m}{sl}
\SetMathAlphabet{\mathsfit}{bold}{\encodingdefault}{\sfdefault}{bx}{n}

\usepackage[utf8]{inputenc} %
\usepackage[T1]{fontenc}    %
\usepackage{hyperref}       %
\usepackage{url}            %
\usepackage{booktabs}       %
\usepackage{amsfonts}       %
\usepackage{nicefrac}       %
\usepackage{microtype}      %
\usepackage[table]{xcolor}  %
\usepackage[most]{tcolorbox}
\tcbuselibrary{theorems, skins}
\usepackage{amssymb}
\usepackage{mathtools}
\usepackage{graphicx}
\usepackage{multirow}
\usepackage{graphicx}
\usepackage{array}
\usepackage{booktabs}
\usepackage{caption}
\usepackage{makecell} 
\usepackage[table]{xcolor}
\usepackage{stackengine}
\usepackage{float}
\usepackage{wrapfig}
\usepackage[normalem]{ulem}
\usepackage{subcaption}
\usepackage{enumitem}

\usepackage{titlesec}
\titlespacing*{\section}
{0pt}    %
{1.0ex}  %
{0.5ex}  %
\titlespacing*{\subsection}
{0pt}
{0.8ex}
{0.4ex}

\titleformat{\paragraph}[runin]
  {\normalfont\normalsize\itshape}
  {\theparagraph)}
  {0.5em}
  {}

\titlespacing*{\paragraph}
  {0pt}{0.5ex}{0.5em}
  
\renewcommand{\theparagraph}{\alph{paragraph}}

\definecolor{linecolor2}{RGB}{230, 234, 217}
\newtheorem{theorem}{Theorem}[section] %

\newcommand{\ie}{\textit{i.e.}}

\newcommand{\observation}{\mathbf{o}}

\newcommand{\actchunk}{\mathbf{A}}

\newcommand{\instruction}{\ell}

\usepackage[most]{tcolorbox}

\definecolor{headerblue}{RGB}{26, 54, 93}
\definecolor{insightyellow}{RGB}{255, 250, 240}
\definecolor{boxborder}{RGB}{254, 235, 200}
\definecolor{proofgreen}{RGB}{240, 255, 244}
\definecolor{proofborder}{RGB}{198, 246, 213}

\newtcolorbox{proofbox}[1]{
    colback=proofgreen, 
    colframe=proofborder, 
    title=\textbf{#1},
    coltitle=black,
    sharp corners
}

\newtcolorbox{insightbox}[1]{
    colback=insightyellow, 
    colframe=boxborder, 
    title=\textbf{#1},
    coltitle=black,
    sharp corners
}

\title{\LARGE \bf
Test-Time Improvement of VLA Policies via Selected Diffusion Noise for Spurious-Robust Action Smoothing
}

\author{%
  \begin{minipage}{\textwidth}
    \centering
     Duc Minh Nguyen$^{1,2}$, Bao-Ngoc Dao$^{1,2}$, Tung M. Luu$^{1,3}$, Binh Gia Nguyen$^{1,2}$, Vinh Tong$^{4,5}$\\
     \vspace{0.03in}
    Anji Liu$^{6}$, Vu N. Duong$^{2}$, Dung D. Le$^{2}$, Daniel Sonntag$^{7,8}$, Trung Le$^{9}$, Ngan Le$^{10}$ \\
     \vspace{0.03in}Jan Peters$^{9}$,  
    An Thai Le$^{1,2}$, Minh Nhat Vu$^{1,2}$, Mathias Niepert$^{4,5}$, Khoa D. Doan$^{2}$\\
    \vspace{0.03in}
    Duy M. H. Nguyen$^{\dagger\,4,5,7}$, Vien Anh Ngo$^{\dagger\,1,2}$
    \\[0.4cm]
    $^{1}$\textmd{VinRobotics} \quad
    $^{2}$\textmd{Center for AI Research, VinUniversity} \quad 
    $^{3}$\textmd{KAIST}\quad $^{4}$\textmd{University of Stuttgart}\quad 
    $^{5}$\textmd{IMPRS-IS}\quad
    $^{6}$\textmd{National University of Singapore} \quad $^{7}$\textmd{DFKI} \quad $^{8}$\textmd{University of Oldenburg},\\
    $^{9}$\textmd{Monash University} \quad $^{10}$\textmd{University of Arkansas} \quad $^{9}$\textmd{TU Darmstadt}\\[0.2cm]
    Project page: \url{https://sdn-vla.github.io/}
  \end{minipage}
}

\begin{document}

\maketitle
\thispagestyle{empty}
\pagestyle{empty}

\begin{abstract}

Diffusion-based Vision-Language-Action (VLA) policies have demonstrated strong generalization in robot manipulation, but their deployment can remain unreliable due to spurious visual reliance and stochastic action generation, leading to grounding failures and kinematically unstable behaviors. We introduce Selected Diffusion Noise (SDN), a training-free method that treats diffusion noise as a controllable degree of freedom for improving test-time reliability. SDN samples multiple candidate action chunks and applies two-stage selection: it first filters candidates based on task-relevant visual grounding using a counterfactual negative set, then selects the kinematically most stable chunk based on jerkRMS. Experiments on three VLA policies across both simulation and real-world environments show improvements of $\textbf{8}$--$\textbf{10\%} $ and $\textbf{18.33\%}$, respectively, demonstrating SDN's effectiveness for test-time improvement.

\end{abstract}

\vspace{0.1in}
\section{INTRODUCTION}

The emergence of generalist robot policies, such as $\pi_0$ \cite{black2024pi_0} and GR00T \cite{bjorck2025gr00t}, demonstrates that scaling diffusion-based Vision-Language-Action (VLA)~\cite{zitkovich2023rt,kim2024openvla} models on large-scale datasets can enable broad multi-task capabilities. However, these policies can remain unreliable during deployment. The stochastic nature of diffusion-based action generation produces multiple candidate action chunks for the same observation and instruction, some of which may exhibit spurious visual reliance or kinematic instability. This raises a practical question: \textit{can we improve the reliability of a frozen VLA policy at test time without retraining it?} Fine-tuning for task or environmental variations requires additional data and computation, while adapting large-scale VLA models can be costly and risk catastrophic forgetting~\cite{lesort2020continual,yang2024preventing}. These limitations motivate a training-free approach that improves action generation directly at inference time without modifying the underlying policy.

Toward that goal, recent work has begun to exploit the stochasticity of diffusion-based action generation through noise steering~\cite{ho2022classifier,wang2026ppguide,liu2026vls,park2025acg}, best-of-$N$ sampling~\cite{wu2025policy,chi2025diffusion}, and non-Gaussian priors~\cite{li2026global,shariatian2024heavy}. While these approaches improve action selection, two practical challenges remain. First, VLA policies can rely on task-irrelevant visual cues, leading to spurious grounding and execution failures~\cite{ye2024spurious,jia2026guidedvla}. For example, in a ``place the carrot onto the plate'' task, a policy may incorrectly infer a successful grasp and move toward the plate while holding nothing, as illustrated in Figure~\ref{fig:visualize}. Second, semantic plausibility does not guarantee physical reliability: a trajectory may reach the intended goal while exhibiting sudden, high-jerk motions that can cause instability and mechanical wear~\cite{black2025real,hao2025disco}.

We address these challenges with \textit{Selected Diffusion Noise (SDN)}, a simple, training-free test-time action selection framework that treats the initial diffusion noise as a \textit{controllable degree of freedom}. Rather than modifying the denoising process, SDN samples multiple candidate action chunks and selects among them in two stages. It first evaluates \textit{visual grounding} against a counterfactual negative set generated by masking task-relevant objects to expose spurious correlations, and then selects the most \textit{kinematically stable} grounded chunk. At each inference step, candidates are generated from both the original and counterfactual observations, where task-relevant objects are masked. These objects are identified in the first frame using an off-the-shelf text-grounded detector (e.g., DINO~\cite{liu2024grounding,ren2024dino}), and their masks are propagated across frames using a video segmentation model (e.g., SAM-2~\cite{ravi2024sam}).

To further improve execution reliability, we apply kinematic refinement to select the final candidate based on its kinematic stability. This favors chunks with lower jerk and reduces abrupt or oscillatory motions that can lead to unstable execution and mechanical wear~\cite{black2025real,hao2025disco}. Unlike prior approaches that manipulate model parameters to encourage smoother outputs~\cite{ahn2026noise,park2025acg}, SDN evaluates the kinematic stability of sampled action chunks directly. This transforms standard stochastic sampling into a lightweight action selection process that can be applied to VLA backbones such as $\pi_0$ and GR00T without parameter updates or auxiliary evaluators. Extensive experiments across three VLA backbones ($\pi_0$, GR00T-N1.5/1.6), two simulation benchmarks (Google Robot and WidowX~\cite{li2024evaluating}), and two real-world ALOHA tasks demonstrate success-rate improvements of up to $8$--$10$ \% in simulation and $18.33$ \% in real-world tasks, without retraining the underlying policies.

\begin{figure*}[!t]
    \centering

    \begin{subfigure}[t]{0.62\linewidth}
        \centering
        \includegraphics[width=\linewidth]{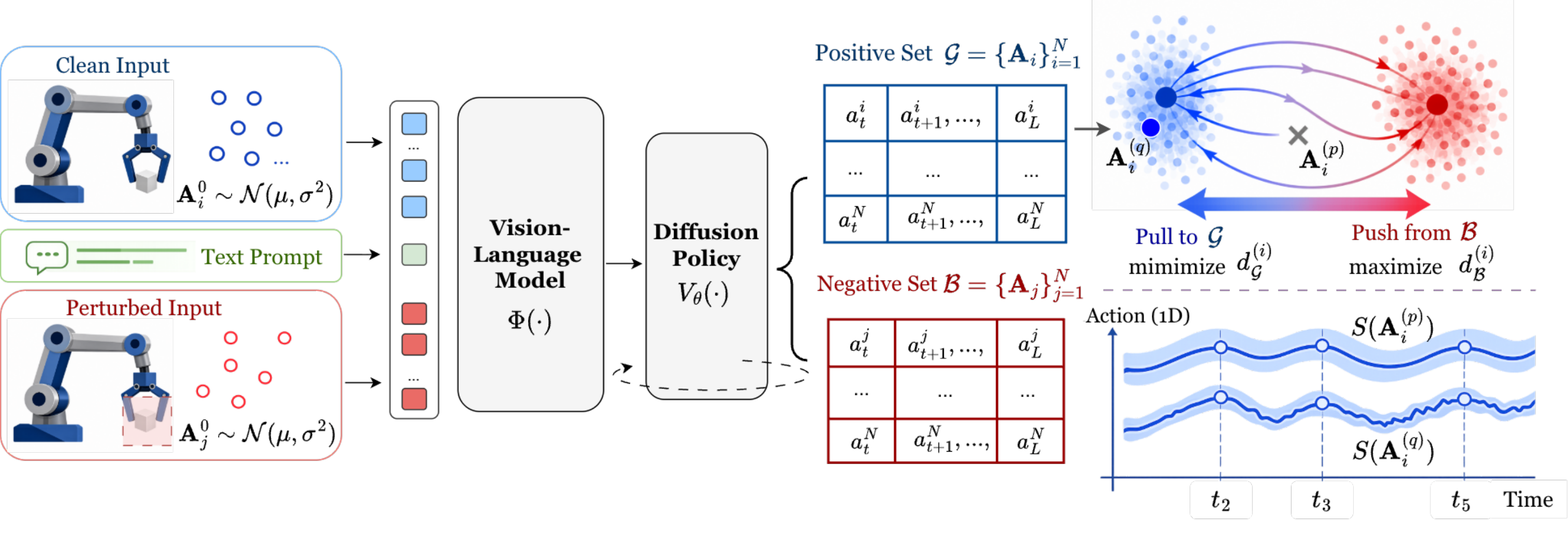} 
        \caption{}
        \label{fig:method_overview}
    \end{subfigure}
    \hfill
    \begin{subfigure}[t]{0.37\linewidth}
        \centering
        \includegraphics[width=\linewidth]{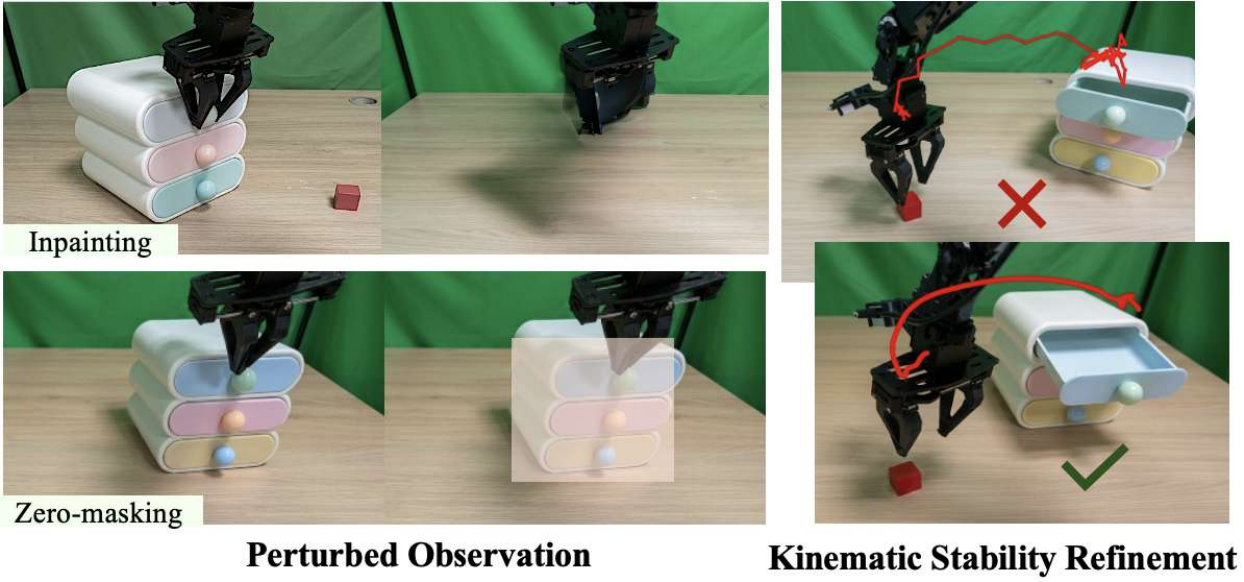}
        \caption{}
        \label{fig:visual_masking}
    \end{subfigure}
    \vspace{-0.1in}

    \caption{\small{(a) SDN refines the VLA action distribution through two-stage selection. \textit{Stage 1 (Grounding Filter)} selects top-$M$ candidates using the $k$-NN grounding score $R_{\text{ground}}^{(i)}=d_{\mathcal{B}}^{(i)}-d_{\mathcal{G}}^{(i)}$. \textit{Stage 2 (Kinematic Stability Refinement)} selects the candidate with minimum jerk. (b) Perturbation  strategy (left) and trajectory comparison with the baseline (right)}}
    \label{fig:overall}
    \vspace{-0.2in}
\end{figure*}

\vspace{0.1in}
\section{RELATED WORK}

\subsection{Learned Latent Steering.}
Recent work has explored steering diffusion latent spaces to adapt pretrained robot policies to new tasks and environments. These approaches include \textit{learned adaptation} through reinforcement learning for novel scenarios~\cite{wagenmaker2025steering,gao2026steervla} and fine-tuning for vision-action alignment~\cite{zhang2025align,liu2026last,im2026latent}. Other directions use VLMs as \textit{high-level planners} or reward providers for low-level diffusion policies~\cite{zheng2022vlmbench,liu2024ok,lee2026roboreward}, or augment policies with retrieval-based memory to leverage global priors and local consistency~\cite{li2026global}. Despite their flexibility, these methods typically rely on additional training, environmental rewards, retrieval databases, or external VLMs~\cite{wagenmaker2025steering,zhang2025align,li2026global,lee2026roboreward}. Such dependencies limit zero-shot improvement of frozen foundation models at test time. 
In contrast, SDN is \textit{self-contained} and \textit{training-free}, enabling test-time improvement without retraining or relying on external reward evaluators.

\subsection{Inference-Time Guidance and Policy Selection.}

Inference-time guidance and trajectory selection provide alternatives to retraining VLA policies. Classifier-Free Guidance (CFG)~\cite{ho2022classifier} steers predictions by contrasting conditional and unconditional fields, while recent best-of-$N$ methods sample multiple candidates and select or aggregate them using external verifiers or ensemble strategies~\cite{chi2025diffusion}. Policy Contrastive Decoding (PCD)~\cite{wu2025policy} addresses visual grounding by contrasting predictions from original and costly inpainting observations, while BYOVLA~\cite{hancock2025run} repeatedly inpaints irrelevant visual cues to suppress their influence on model attention. ACG~\cite{park2025acg} instead mitigates jerky motions through designed perturbations of the generation process rather than directly evaluating action smoothness. VLA-Pilot~\cite{li2026towards} further refines actions through evolutionary optimization with an external reasoning module for reward generation. In contrast, SDN exploits the noise space to generate diverse action-chunk candidates and selects them using two complementary criteria, visual grounding and kinematic stability, enabling training-free improvements in both grounding and execution reliability.

\subsection{Robustness and Smoothness in Robotic Diffusion.}

Robotic diffusion policies can suffer from both unreliable visual features and kinematic instability. DisDP~\cite{vanjani2025disdp} improves robustness by disentangling shared and private modality embeddings to suppress environmental noise, while Colored Gaussian Noise~\cite{guo2024smooth} promotes temporally coherent trajectories by modifying the diffusion noise structure. Reactive Diffusion Policy (RDP)~\cite{xue2025reactive} combines high-level planning with tactile feedback~\cite{xue2025reactive}. These approaches generally require training-time modifications or specialized hardware~\cite{vanjani2025disdp,xue2025reactive}. In contrast, SDN treats kinematic stability as a test-time selection criterion, sampling multiple action chunks and selecting those with lower jerk without retraining or additional hardware.

\vspace{0.05in}
\section{PRELIMINARIES}
We model language-conditioned imitation learning as finding a VLA policy $\pi_\theta$ that maximizes the log-likelihood of expert action chunks $\mathbf{A}_t \in \mathbb{R}^{L \times D}$ from a dataset $\mathcal{D}$, where $L$ is the chunk length and $D$ the action dimension. At each timestep $t$, the agent receives an observation $\observation_t$ consisting of multiple RGB images, the robot’s joint state, and a language instruction $\ell$. The VLA architecture comprises a VLM encoder $\Phi(\cdot)$ that extracts multi-modal context $\mathbf{e} = \Phi(\mathbf{o}_t, \ell)$ and a diffusion transformer $V_\theta(\cdot)$ that predicts a time-dependent vector field.

Using Flow Matching (FM)~\cite{lipman2024flow, lipman2023flow}, the model learns to transport Gaussian noise $\mathbf{A}^0 \sim \mathcal{N}(\mathbf{0}, \mathbf{I})$ to an expert action $\mathbf{A}^1$ via the probability path $\mathbf{A}^\tau = (1-\tau)\mathbf{A}^0 + \tau\mathbf{A}^1$, where $\tau \in [0, 1]$ is the flow matching time step. The policy is trained by regressing $V_\theta(\cdot)$ to the target velocity $\mathbf{A}^1 - \mathbf{A}^0$:
\begin{equation}
\small
\mathcal{L}_{\text{FM}}(\theta) = \mathbb{E}_{\tau, \mathbf{A}^0, \mathbf{A}^1, \mathbf{e}} \left[ \| V_\theta(\mathbf{A}^\tau, \tau, \mathbf{e}) - (\mathbf{A}^1 - \mathbf{A}^0) \|^2 \right].
\end{equation}
We define $\{\tau_k\}_{k=0}^K$ as a discretization of the time interval $[0,1]$, with $\Delta\tau = \tau_{k+1}-\tau_k$ \cite{tonglearning}. At inference, actions are generated by solving the ODE $\frac{d\mathbf{A}^\tau}{d\tau} = V_\theta(\mathbf{A}^\tau, \tau, \mathbf{e})$ via the Euler update
$\mathbf{A}^{\tau_{k+1}} = \mathbf{A}^{\tau_k} + \Delta\tau\, V_\theta(\mathbf{A}^{\tau_k}, \tau_k, \mathbf{e})$. This defines a deterministic mapping $f_\theta:\mathbf{A}^0 \mapsto \mathbf{A}^1$, enabling us to treat $\mathbf{A}^0$ as a controllable degree of freedom.

\vspace{0.1in}
\noindent
\textbf{Classifier-Free Guidance (CFG)~\cite{ho2022classifier}.}
To improve adherence to the instructions, guidance can be applied during ODE integration by modifying the velocity vector field. Let
$
V_\theta^{\text{cond}}(\mathbf{A}^\tau, \tau, \mathbf{e})
\quad \text{and} \quad
V_\theta^{\text{uncond}}(\mathbf{A}^\tau, \tau, \mathbf{e}^{\text{uncond}})
$
denote the conditional and unconditional vector fields, where the latter is often obtained by masking the instruction (\ie, $\ell = \emptyset$) \cite{reuss2023goal}. The guided field is defined as
\begin{equation}
\small
\begin{aligned}
& V_\theta^{\text{guided}}(\mathbf{A}^\tau,\tau,\mathbf{e}) = \quad V_\theta^{\text{cond}}(\mathbf{A}^\tau,\tau,\mathbf{e}) \\
&  + w \Big[
V_\theta^{\text{cond}}(\mathbf{A}^\tau,\tau,\mathbf{e})  - V_\theta^{\text{uncond}}
(\mathbf{A}^\tau,\tau,\mathbf{e}^{\text{uncond}})
\Big],
\end{aligned}
\end{equation}
where $w \ge 0$ is the guidance scale. This steers the sampling process toward regions of the action space that are more consistent with the instruction.

\vspace{0.05in}
\section{METHODOLOGY}

We depict the overall design of SDN in Figure~\ref{fig:method_overview}. We first discuss the motivation and then present our two-stage selection algorithm.

\subsection{Motivation: Two Sources of Action Unreliability}

VLA policies can still produce unreliable action trajectories at deployment time. We identify two complementary failure modes. First, policies may rely on spurious visual cues, leading to visual grounding failures such as ``ghost-grasping''. Second, they may generate actions that are semantically plausible but physically unstable, exhibiting high-jerk or oscillatory motions. To formalize this intuition, we consider a conceptual decomposition:
\begin{equation}
\small
\pi_\theta(\actchunk|\observation,\instruction)
=
(1-\epsilon)p^*(\actchunk|\observation,\instruction)
+
\epsilon p_{\mathrm{fail}}(\actchunk|\observation,\instruction),
\end{equation}
where $\epsilon\in[0,1]$ denotes the context-dependent probability of unreliable behavior. Here, $p_{\mathrm{fail}}$ captures incorrect or unstable actions arising from approximation errors and spurious correlations. Such behaviors need not be random outliers and may receive non-negligible probability under the same conditioning as expert-like behaviors. Consequently, different initial noise samples can produce substantially different action chunks, motivating inference-time selection of reliable candidates.

\noindent
\textbf{From Local Guidance to Global Selection.}
Guidance steers generative models toward desirable outputs \cite{ho2022classifier,sanchez2023stay}. Classifier-Free Guidance (CFG) does so by modifying the conditional velocity field using an unconditional field, which does not explicitly capture structured robotic failures such as spurious-reliance grasps or unstable motions.

A more direct approach would introduce a failure-aware vector field $V_\theta^{\mathrm{fail}}$ and guide sampling away from it \cite{sanchez2023stay,liu2022compositional,azarian2024segmentation}:
\begin{equation}
V_\theta^{\mathrm{guided}}
=
V_\theta^{\mathrm{cond}}
+
w\big(
V_\theta^{\mathrm{cond}}
-
V_\theta^{\mathrm{fail}}
\big).
\end{equation}
However, failure trajectories are rarely labeled, learning an additional failure model adds overhead \cite{jang2026verifier}, and guidance can be sensitive to its scale \cite{so2025improving,sun2025latent}. Instead, we adopt a global selection perspective: sample multiple action chunks from different initial noises and select select the most suitable candidate, avoiding stepwise intervention during generation.

\subsection{The Selected Diffusion Noise Framework}

We generate multiple candidate action chunks from distinct noise seeds and select the most reliable candidate. This reframes inference as a global search over action predictions rather than step-by-step steering. Our training-free framework performs selection in two stages: it first filters candidates based on visual grounding and then refines them based on physical stability.

\textbf{Stage 1: Grounding Filter.}
The first stage aims to identify ``grounded'' initial noise seeds via a contrastive evaluation. Intuitively, we prefer samples that are more consistent with task-aligned behavior while being far from typical failure patterns.

\noindent
\textbf{Constructing selection sets.}
At each inference step, we approximate this idea using two sets of action chunk candidates generated from the same policy under different conditions:
\begin{itemize}%
    \item \textbf{Positive set ($\mathcal{G}$):} Candidates generated from the original observation and instruction,
    $\mathcal{G} = \{\actchunk_{i}\}_{i=1}^N \sim \pi_\theta(\cdot \mid \observation, \instruction)$.

    \item \textbf{Negative set ($\mathcal{B}$):} Candidates generated from a counterfactual observation,
    $\mathcal{B} = \{\actchunk_j\}_{j=1}^N \sim \pi_\theta(\cdot \mid \observation_{\text{neg}},\instruction)$.
     $\observation_{\text{neg}}$ is constructed by perturbing the raw observation to suppress task-relevant visual evidence and induce reliance on spurious cues. We evaluate two perturbations: (1) \textit{Inpainting:} Removes the target object while preserving the surrounding visual context, (2) \textit{Zero bounding-box masking:} Removes the target region using zero masking as a simpler and more efficient alternative.
\end{itemize}

\noindent
\textbf{Contrastive Grounding Score.} A grounded action chunk should remain consistent under the original observation while becoming unlikely under corrupted ones. We formulate grounding as a contrastive, distanced-based ranking problem between a task-aligned distribution $p_{\mathcal{G}}$ and a failure distribution $p_{\mathcal{B}}$. Given a candidate action chunk $\actchunk_i$, we define its score as the log density ratio $R(\actchunk_i) = \log p_{\mathcal{G}}(\actchunk_i) - \log p_{\mathcal{B}}(\actchunk_i).
$

Since both densities are unknown, we estimate them from finite samples using $k$-nearest neighbor (kNN) density estimation~\cite{loftsgaarden1965nonparametric,perez2008kullback}. Concretely, let $\Omega_i^{\mathcal{G}}$ and $\Omega_i^{\mathcal{B}}$ denote the indices of the $k$ nearest neighbors of $\actchunk_i$ in $\mathcal{G} \setminus \{\actchunk_i\}$ and $\mathcal{B}$, respectively. We define the corresponding local scale estimates:
\begin{equation}
\small
d_{\mathcal{G}}^{(i)} = \frac{1}{k} \sum_{j \in \Omega_i^{\mathcal{G}}} \|\actchunk_i - \actchunk_j\|_2,
\qquad
d_{\mathcal{B}}^{(i)} = \frac{1}{k} \sum_{j \in \Omega_i^{\mathcal{B}}} \|\actchunk_i - \actchunk_j\|_2.
\end{equation}
Under standard kNN density estimation, local density is inversely proportional to the neighborhood volume, which is approximated (up to constants) by these average distances~\cite{loftsgaarden1965nonparametric,perez2008kullback}. This yields:
\begin{equation}
\small
\log p_{\mathcal{G}}(\actchunk_i) \approx - \log d_{\mathcal{G}}^{(i)} + C,
\quad
\log p_{\mathcal{B}}(\actchunk_i) \approx - \log d_{\mathcal{B}}^{(i)} + C.
\end{equation}
Substituting into the density-ratio objective gives
$R(\actchunk_i) \propto \log d_{\mathcal{B}}^{(i)} - \log d_{\mathcal{G}}^{(i)}.$ For efficiency and robustness, we use a monotonic surrogate that preserves the induced ranking:
\begin{equation}
\small
R_{\text{ground}}^{(i)} = d_{\mathcal{B}}^{(i)} - d_{\mathcal{G}}^{(i)}. \label{eq:full_ground}
\end{equation}
If the target object is invisible, the perturbation is disabled and the negative set is empty. The grounding score is then computed solely from the positive set:
\begin{equation}
\small
R_{\text{ground}}^{(i)} = - d_{\mathcal{G}}^{(i)}. \label{eq:positive_ground}
\end{equation}
We select the top-$M$ candidates ranked by $R_{\text{ground}}$ for kinematic refinement.

\vspace{0.05in}
\textbf{Stage 2: Kinematic Stability Refinement.}
To ensure physical feasibility, we further evaluate the remaining $M$ candidates for kinematic reliability. High-jerk or oscillatory motions correspond to abrupt changes in acceleration across consecutive steps, which often lead to unstable or unsafe robot behaviors. To better assess temporal consistency, we generate action chunks that are longer than the executed horizon. These extended action chunks expose delayed oscillations and accumulated instability that may not be observable within short action windows. We quantify smoothness using the JerkRMS metric~\cite{Flash1688}. For a local candidate action chunk $\actchunk_i = [a_1^{(i)}, a_2^{(i)}, \dots, a_L^{(i)}]$, where each $a_t^{(i)}$ denotes the position-based action at timestep $t$, we compute:
\begin{equation}
\begin{aligned}
\small
S(\actchunk_i)
&=
\sqrt{
\frac{1}{L-3}
\sum_{t=1}^{L-3}
\left\|
\Delta a_t^{(i)}
\right\|_2^2
}
\\
\text{where}\qquad
\Delta a_t^{(i)}
&=
a_{t+3}^{(i)}
- 3a_{t+2}^{(i)}
+ 3a_{t+1}^{(i)}
- a_t^{(i)}.
\end{aligned}
\end{equation}

This third-order finite difference captures discrete jerk, and $S(\actchunk_i)$ corresponds to its RMS magnitude over time. Minimizing $S(\actchunk_i)$ suppresses candidates that are dynamically unstable, even if they achieve the correct task outcome. We therefore select the final action as:
\begin{equation}
\actchunk^* = \arg\min_{\actchunk_i \in \mathcal{G}_M} S(\actchunk_i).
\end{equation}

The pseudocode is given in Algorithm~\ref{alg:sdn}, and a theoretical analysis is provided in Appendix~\ref{thm:grounding_consistency}.

\begin{figure*}[!t]
    \centering
    \begin{minipage}[t]{0.68\linewidth}
        \centering
        \resizebox{0.9\linewidth}{!}{
        \vspace{-0.15in}
        \includegraphics[width=\linewidth]{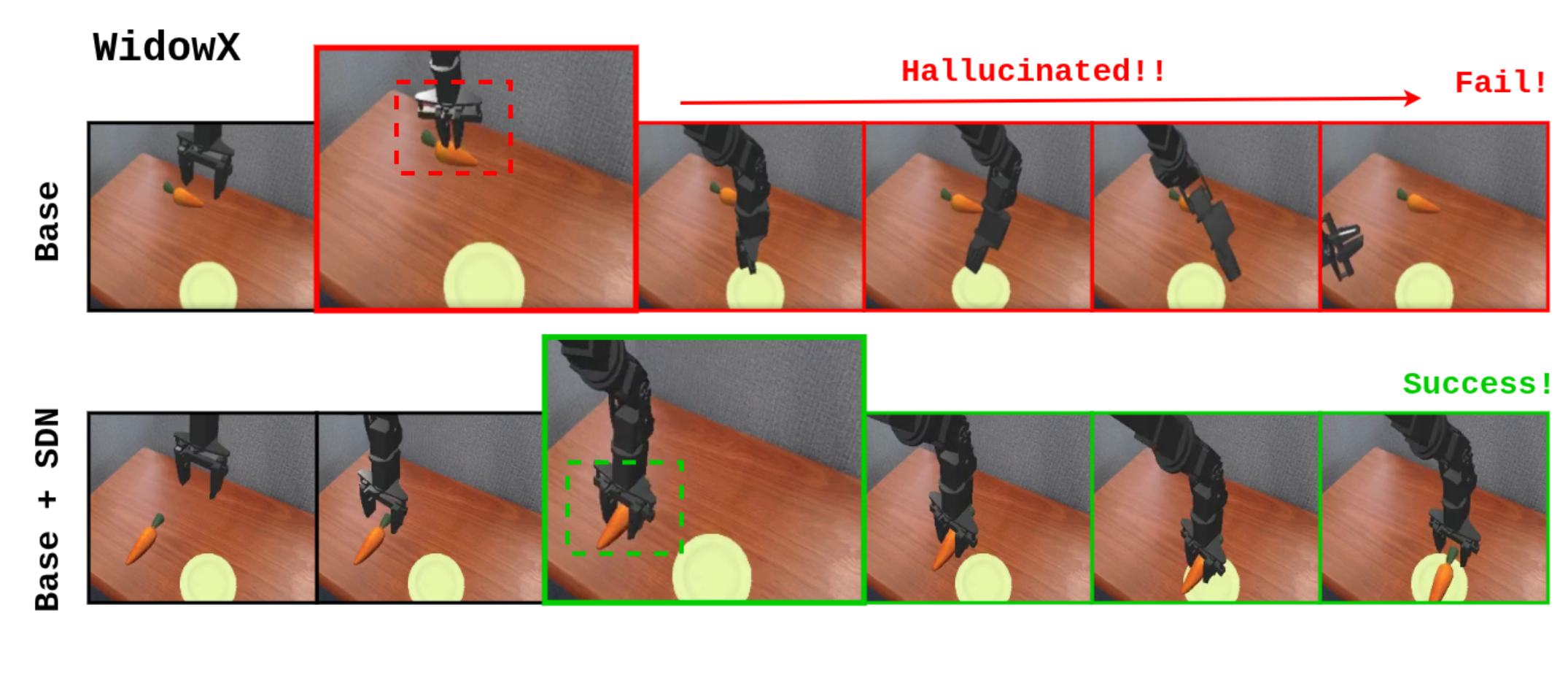}}
        \label{fig:guidance_methods}
    \end{minipage}
    \begin{minipage}[t]{0.3\linewidth}
        \centering
        \includegraphics[width=0.95\linewidth]{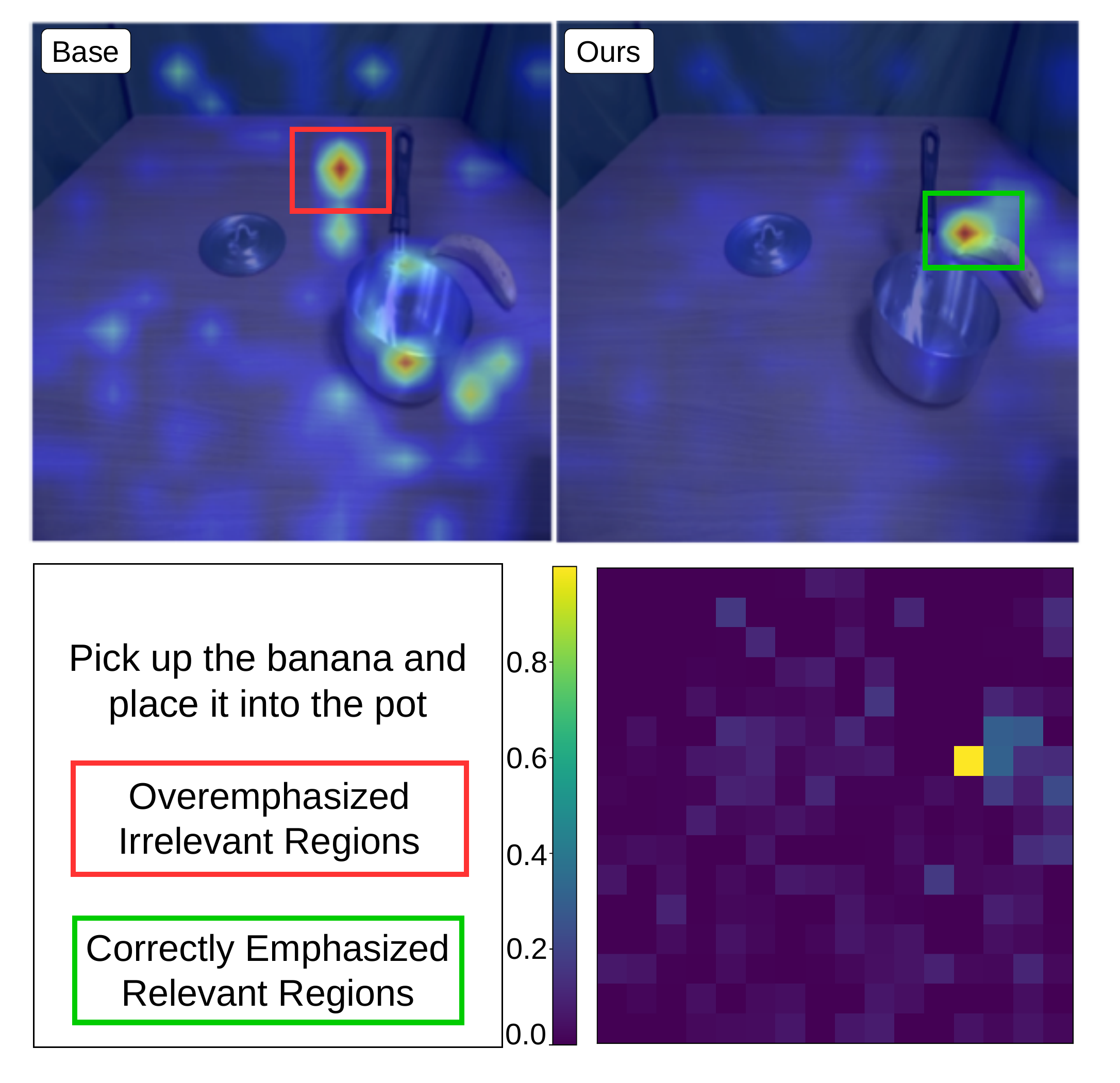}
        \label{fig:fig2}
    \end{minipage}
    \vspace{-0.1in}
    \caption{Visualizing hallucination mitigation and attention grounding. (\textbf{Left}): Base VLA continues the task after a failed grasp, while SDN alleviates this failure. (\textbf{Right}): Base VLA attends to irrelevant table regions (top-left), whereas SDN focuses on the target object (top-right). Bottom: Attention matrix showing a shift toward sparse, task-relevant features.}
    \vspace{-0.1in}
    \label{fig:visualize}
\end{figure*}
\begin{algorithm}[t]
\small
\caption{Selected Diffusion Noise (SDN) Inference}
\label{alg:sdn}
\begin{algorithmic}
\REQUIRE Policy $\pi_\theta$, observation $\textbf{o}$, instruction $\instruction$
\REQUIRE Number of samples $N$, filter size $M$

\STATE Initialize $\mathcal{G} \gets \emptyset$, $\mathcal{B} \gets \emptyset$

\FOR{$i = 1$ to $N$}
    \STATE Sample $\actchunk_i^0 \sim \mathcal{N}(0, I)$
    \STATE Generate $\actchunk_i \sim \pi_\theta(\cdot \mid \textbf{o}, \instruction)$
    \STATE $\mathcal{G} \gets \mathcal{G} \cup \{\actchunk_i\}$
\ENDFOR

\IF{target object is detected}
    \FOR{$i = 1$ to $N$}
        \STATE Generate $\tilde{\actchunk}_i \sim \pi_\theta(\cdot \mid \textbf{o}_{\mathrm{neg}}, \instruction)$
        \STATE $\mathcal{B} \gets \mathcal{B} \cup \{\tilde{\actchunk}_i\}$
    \ENDFOR

    \STATE Compute $R_{\mathrm{ground}}^{(i)}$ for all $\actchunk_i \in \mathcal{G}$ and $\tilde{\actchunk}_i \in \mathcal{B}$ using Equation~\ref{eq:full_ground}
    \STATE $\mathcal{G}_M \gets \text{TopM}(\mathcal{G})$ based on $R_{\mathrm{ground}}$
    
    \FOR{$\actchunk_i \in \mathcal{G}_M$}
        \STATE Compute $S(\actchunk_i)$
    \ENDFOR
    
    \STATE $\actchunk^* \gets \arg\min_{\actchunk_i \in \mathcal{G}_M} S(\actchunk_i)$
\ELSE
    \STATE Compute $R_{\mathrm{ground}}^{(i)}$ for all $\actchunk_i \in \mathcal{G}$ using Equation~\ref{eq:positive_ground}
    \STATE $\actchunk^* \gets \arg\min_{\actchunk_i \in \mathcal{G}} d_g^{(i)}
    \quad$ \COMMENT{positive grounding only}
\ENDIF

\STATE \textbf{return} $\actchunk^*$
\end{algorithmic}
\end{algorithm}

\section{EXPERIMENTS}
\subsection{\textbf{Experimental Setup}}

\paragraph{\texttt{Benchmarks.}}
We evaluate SDN in both simulation and the real world. In simulation, we use SIMPLER~\cite{li2024evaluating}, following \cite{wu2025policy} for task selection, with five tasks on the \textit{Google Robot} and four on the \textit{WidowX arm}, each evaluated over 50 rollouts. For real-world evaluation, we deploy SDN on ALOHA for two manipulation tasks: \textit{placing a cube into a drawer (Cube2Drawer)} and \textit{placing a banana into a pot (Banana2Pot)}, each evaluated over 20 rollouts.

\noindent
\paragraph{\texttt{Baselines.}}
We compare SDN with vanilla VLA policies and five training-free guidance baselines: \textit{Classifier-Free Guidance (CFG)~\cite{ho2022classifier}}, which contrasts conditional and unconditional action fields to strengthen language adherence; \textit{Bring Your Own VLA (BYOVLA)~\cite{hancock2025run}}, which modifies task-irrelevant visual regions to reduce sensitivity to distractors; \textit{Action Coherent Guidance (ACG)~\cite{park2025acg}}, which disrupts action coherence by replacing attention matrices with identity matrices; \textit{White Noise Guidance (WNG)}, which injects white noise ($\sigma=1$) into action features to disrupt temporal coherence, following its use as a baseline in \cite{park2025acg}; and \textit{Policy Contrastive Decoding (PCD)~\cite{wu2025policy}}, which contrasts predictions from original and object-masked observations to mitigate spurious visual reliance.

\noindent
\paragraph{\texttt{Implementation Details.}} We evaluate SDN on three diffusion-based VLA models: $\pi_0$ and \textsc{GR00T N1.6} in simulation, and \textsc{GR00T N1.5} in the real world. SDN is training-free and operates only at inference time without modifying the underlying model. For simulation, we use the provided pretrained checkpoints; for real-world experiments, we fine-tune \textsc{GR00T N1.5} for our robotic platform. All SIMPLER experiments use a single environment instance. We use zero-masking as the primary perturbation to construct the negative set due to its simplicity and efficiency, and evaluate robustness to inpainting in RQ4 (\ref{exp:sdn_vs_pcd}). SDN samples $N=12$ action chunks from different noise seeds. Stage 1 uses kNN-based contrastive scoring with $k\in\{6,10\}$ to retain the top-$M\in\{3,5\}$ candidates that are closest to the grounded and farthest from the perturbed action distributions. Stage 2 ranks the remaining candidates using JerkRMS \cite{Flash1688} over extended action chunks, with the number of extended chunks searched over $\{4,5,10\}$. For negative-set construction, we use Grounding-DINO \cite{ren2024dino} to detect task-relevant objects and SAM2 \cite{ravi2024sam} to track them across timesteps. We then mask their bounding-box regions with zero-valued pixels or inpaint objects with LaMa \cite{suvorov2021resolutionrobustlargemaskinpainting} to remove task-relevant object information.

\subsection{{\textbf{RQ1. Does SDN actually improve VLA performance?}}}

\paragraph{\texttt{Simulation Experiments}.}

Table~\ref{tab:simpler_gr00t} presents the performance of SDN on the SIMPLER benchmark using \textsc{GR00T N1.6}. On the Google Robot platform, SDN consistently surpasses the base model, achieving significant gains, most notably tripling the success rate on the ``Open Drawer'' task (from 4.0\% to 14.0\%). This success generalizes to the WidowX platform, where SDN enhances performance on challenging tasks like Spoon Towel and Stack Cube, resulting in a 54.5\% average success rate (a 5.1\% absolute improvement over the baseline). Crucially, SDN exhibits a {\textit{do-no-harm property}}: while baselines like CFG, BYOVLA and PCD occasionally degrade performance on specific tasks, SDN provides non-negative gains across the entire benchmark, proving its reliability as a training-free inference-time wrapper. We provide a visualization in Figure~\ref{fig:visualize} of roll-out actions. 

Table~\ref{tab:simpler_pi0} further validates SDN’s efficacy on $\pi_0$, yielding a substantial 8.7\% average gain. While baselines like ACG, BYOVLA and WNG often degrade the vanilla model, SDN consistently improves success rates, particularly on high-precision WidowX tasks such as Spoon Towel (90.0\%) and Stack Cube (78.0\%). These results confirm that SDN is \textit{architecture-agnostic}, providing robust, additive improvements to foundation-scale models. By achieving the highest average success rate across all nine tasks, SDN establishes itself as a superior, \emph{do-no-harm} solution for inference-time VLA adaptation.
\newcolumntype{C}[1]{>{\centering\arraybackslash}p{#1}}
\newcommand{\gain}[2]{%
  $\hspace{0.22in} \textbf{#1}_{\textcolor{green!60!black}{\uparrow#2}}$%
}
\newcommand{\second}[2]{%
  $\hspace{0.22in} \underline{#1}_{\textcolor{green!60!black}{\uparrow#2}}$%
}
\begin{table*}

\renewcommand{\arraystretch}{0.85}\setlength{\tabcolsep}{1pt}
\setlength{\tabcolsep}{1pt}
\centering
\caption{\small{\textbf{Evaluation of GR00T N1.6 on SIMPLER}. We report the average success rate over 50 rollouts per task. The best results are shown in \textbf{bold}, the second-best results are \underline{underlined}.}}
\label{tab:simpler_gr00t}
\vspace{-0.08in}

\resizebox{0.9\linewidth}{!}{
\begin{tabular}{c|ccccc|cccc|c}
\toprule
\multirow{3}{*} {\textbf{Method}} & 
\multicolumn{5}{c|}{\textbf{GoogleRobot}} & 
\multicolumn{4}{c|}{\textbf{WidowX}}&Task(All) \\
\cmidrule(lr){2-6} \cmidrule(lr){7-10} \cmidrule(lr){11-11}

& \textbf{\makecell{Close\\Drawer}}
& \textbf{\makecell{Move\\Near}}
& \textbf{\makecell{Open\\Drawer}}
& \textbf{\makecell{Pick Coke\\Can}}
& \textbf{\makecell{Apple\\Drawer}}
& \textbf{\makecell{Carrot\\Plate}}
& \textbf{\makecell{Eggplant\\Basket}}
& \textbf{\makecell{Spoon\\Towel}}
& \textbf{\makecell{Stack\\Cube}} & \makecell{Avg.} \\

\midrule
 GR00T\,N1.6 & 40.0 & 76.0 & 4.0 & 90.0 & 18.0 & 60.0 & 94.0 & 58.7 & 4.7 & 49.5 \\
 \midrule
CFG           & \underline{42.0} & \underline{80.0} & \underline{8.0} & 92.0 & 10.0 & 40.0 & 82.0 & 46.0 & 4.0 & 44.9 \\
WNG           & 40.0 & 78.0 & 6.0 & 94.0 & 20.0 & \textbf{68.0} & 78.0 & \underline{64.0} & \textbf{8.0} & 50.7\\
BYOVLA & 42.0 & 76.0 & 10.0 & \textbf{98.0} & 12.0 & 52.0 & 78.0 & 60.0 & 0.0 & 47.6 \\
ACG           & 40.7 & 79.3 & 7.3 & \underline{95.3} & 20.0 & 60.0 & \underline{93.0} & 61.0 & 4.0 & \underline{51.2}\\
PCD           & \underline{42.0} & \underline{80.0} & 4.0 & 92.0 & \underline{12.0} & \underline{62.0} & 74.0 & 60.0 & 6.0 & 48.0 \\
\rowcolor{linecolor2} \textbf{SDN} & \gain{44.0}{4.0} & \hspace{0.01in} \gain{90.0}{26.0} & \gain{14.0}{10.0} & \hspace{0.21in}$94.0_{\textcolor{green!60!black}{\uparrow 4.0}}$ & \gain{20.0}{2.0} & \hspace{0.21in}$60.0_{\textcolor{green!60!black}{\uparrow 0.0}}$ & \gain{94.0}{0.0} & \gain{68.0}{9.3} & \second{6.0}{1.3} & \gain{54.5}{5.1}\\
\bottomrule
\end{tabular}
}
\vspace{-0.1in}
\end{table*}

\begin{table*}

\renewcommand{\arraystretch}{0.85}\setlength{\tabcolsep}{1pt}
\setlength{\tabcolsep}{1pt}
\centering
\caption{\small{\textbf{Evaluation of $\boldsymbol{\pi_0}$ on SIMPLER}. We report the average success rate over 50 rollouts per task. The best results are shown in \textbf{bold}, the second-best results are \underline{underlined}.}}
\label{tab:simpler_pi0}
\vspace{-0.08in}
\resizebox{0.9\linewidth}{!}{
\begin{tabular}{c|ccccc|cccc|c}
\toprule
\multirow{3}{*} {\textbf{Method}} & 
\multicolumn{5}{c|}{\textbf{GoogleRobot}} & 
\multicolumn{4}{c|}{\textbf{WidowX}}&Task(All) \\
\cmidrule(lr){2-6} \cmidrule(lr){7-10} \cmidrule(lr){11-11} 

& \textbf{\makecell{Close\\Drawer}}
& \textbf{\makecell{Move\\Near}}
& \textbf{\makecell{Open\\Drawer}}
& \textbf{\makecell{Pick Coke\\Can}}
& \textbf{\makecell{Apple\\Drawer}}
& \textbf{\makecell{Carrot\\Plate}}
& \textbf{\makecell{Eggplant\\Basket}}
& \textbf{\makecell{Spoon\\Towel}}
& \textbf{\makecell{Stack\\Cube}}
& \makecell{Avg.} \\

\midrule

$\pi_0$
& 75.7
& 67.3
& 38.0
& 84.0
& 17.0
& 58.0
& 86.0
& 80.7
& 68.7
& 63.9
\\

\midrule

CFG
& \underline{77.0}
& 51.0
& 45.0
& 71.0
& 26.0
& 49.0
& 69.0
& 75.0
& 39.0
& 55.8
\\

WNG
& \underline{77.0}
& 60.0
& 29.0
& 84.0
& 19.0
& 63.0
& 83.0
& 80.0
& 45.0
& 60.0
\\
BYOVLA 
&66.0 
&42.0
&0.0
&\textbf{92.0}
&0.0
&46.0
&20.0
&56.0
&38.0
& 40.0\\

ACG
& 70.0
& 57.0
& 40.0
& 82.0
& 17.0
& \textbf{65.0}
& \underline{87.0}
& 82.0
& 53.0
& 61.4
\\

PCD
& 75.0
& \underline{72.3}
& \textbf{56.3}
& 88.0
& \textbf{27.3}
& 59.7
& \underline{87.0}
& \underline{84.0}
& \underline{77.0}
& \underline{69.6}
\\

\rowcolor{linecolor2}
\textbf{SDN}
& \hspace{0.01in} \gain{87.0}{11.3}
& \hspace{0.01in} \gain{78.0}{10.7}
& \hspace{0.01in} \second{50.0}{12.0}
& \second{89.0}{5.0} 
& \second{26.0}{9.0}
& \second{63.0}{5.0}
& \gain{92.0}{6.0}
& \gain{90.0}{9.3}
& \gain{78.0}{9.3}
& \gain{72.6}{8.7}
\\

\bottomrule
\end{tabular}
}
\vspace{-0.15in}
\end{table*}

\paragraph{\texttt{Real-world Experiments.}}
Table \ref{tab:realworld_results} evaluates SDN on the ALOHA robot using \textsc{GR00T N1.5}, demonstrating generalization in real-world settings. SDN achieves an average success rate of 48.33\%, a substantial 18.33\% absolute improvement over the vanilla model and significantly outperforming baselines like CFG, ACG, and PCD. 

Beyond success rates, SDN is the only method that simultaneously improves motion quality. While all other guidance baselines substantially increase JerkRMS~\cite{Flash1688} and Action Total Variation (ATV)~\cite{rudin1992nonlinear}, SDN reduces both metrics below the vanilla baseline, with Jerk decreasing from 12.45\% to 11.8\%. This contrast is particularly pronounced for ACG, which exhibits a 114\% increase in JerkRMS, suggesting that manually designed negative directions are highly sensitive and may generalize poorly in practice. Such heuristics can push the policy toward out-of-distribution regions, inducing high-frequency oscillations. In contrast, SDN first enforces task-relevant visual grounding through contrastive evaluation against a visually driven negative set and then selects the final action chunk based on kinematic feasibility, yielding smoother motions that are suited for reliable real-world deployment.

\newcommand{\up}[1]{\textcolor{green!60!black}{\scriptsize$\uparrow$#1}}
\newcommand{\down}[1]{\textcolor{red!60!black}{\scriptsize$\downarrow$#1}}

\newcommand{\gup}[1]{\textcolor{red!60!black}{\scriptsize$\uparrow$#1}}
\newcommand{\gdown}[1]{\textcolor{green!60!black}{\scriptsize$\downarrow$#1}}

\renewcommand{\arraystretch}{0.9}
\setlength{\tabcolsep}{3pt}

\begin{table}[t]
\centering
\caption{\small{\textbf{Performance on real-world tasks.} Top: task success rate (\%). Bottom: motion quality (lower is better).}}
\label{tab:realworld_results}
\vspace{-0.05in}
\begin{subtable}[t]{\columnwidth}
\centering

\resizebox{0.95\columnwidth}{!}{
\begin{tabular}{c|cc|c}
\toprule
\multicolumn{4}{c}{Task Success (\%)} \\
\midrule
\textbf{Method} & \textbf{Cube2Drawer} & \textbf{Banana2Pot} & \textbf{Avg} \\
\midrule
GR00T N1.5 & 30.00 & 30.00 & 30.00 \\
\midrule
CFG & 40.00 {\up{10.00}} & 33.33 {\up{3.33}} & 36.67 {\up{6.67}} \\
ACG & 36.67 {\up{6.67}} & 46.67 {\up{16.67}} & 41.67 {\up{11.67}} \\
PCD & 40.00 {\up{10.00}} & 43.33 {\up{13.33}} & 41.67 {\up{11.67}} \\
\midrule
\rowcolor{linecolor2}
\textbf{SDN - Smoothness} & 50.00 {\up{20.00}} & 43.33 {\up{13.33}} & 46.67 {\up{16.67}} \\
\rowcolor{linecolor2}
\textbf{SDN - Grounding} & 50.00 {\up{20.00}} & 36.67 {\up{6.67}} & 43.33 {\up{13.33}} \\
\rowcolor{linecolor2}
\textbf{SDN - Full}
& \textbf{53.33 {\up{23.33}}}
& \textbf{43.33 {\up{13.33}}}
& \textbf{48.33 {\up{18.33}}} \\
\bottomrule
\end{tabular}
}
\end{subtable}
\hfill
\begin{subtable}[t]{\columnwidth}
\centering
\resizebox{\columnwidth}{!}{
\begin{tabular}{c|cc|cc}
\toprule
\multicolumn{5}{c}{Motion Quality (Jerk $\downarrow$, ATV ($\times 10^2$)$\downarrow$)} \\
\midrule
\textbf{Method} & \multicolumn{2}{c|}{\textbf{Cube2Drawer}} &
\multicolumn{2}{c}{\textbf{Banana2Pot}} \\
& Jerk & ATV & Jerk & ATV \\
\midrule
GR00T N1.5
& 12.45 & 4.87 & 7.07 & 4.80 \\
\midrule
CFG
& 12.79 {\gup{2.7\%}}
& 4.90 {\gup{0.6\%}}
& 7.18 {\gup{1.5\%}}
& 4.82 {\gup{0.4\%}} \\

ACG
& 18.1 {\gup{45.4\%}}
& 6.00 {\gup{23.2\%}}
& 15.17 {\gup{114\%}}
& 6.00 {\gup{25.0\%}} \\

PCD
& 12.47 {\gup{0.1\%}}
& 4.88 {\gup{0.2\%}}
& 7.08 {\gup{0.1\%}}
& 4.80 {\gup{0.0\%}} \\

\midrule
\rowcolor{linecolor2}
\textbf{SDN}
& \textbf{11.80 {\gdown{5.2\%}}}
& \textbf{4.80 {\gdown{1.4\%}}}
& \textbf{6.54 {\gdown{7.5\%}}}
& \textbf{4.70 {\gdown{2.1\%}}} \\
\bottomrule
\end{tabular}
}
\end{subtable}

\end{table}

\subsection{{\textbf{RQ2. Why is the negative set useful for Stage 1?}}}

\begin{table}[t]
    \centering
    \setlength{\tabcolsep}{3pt}
    \caption{\small{\textbf{Performance of $\pi_0$ and GR00T N1.6 with Masked Objects and Inpainting.} Success rate (\%).}}
    \label{tab:base_masking_image}

    \resizebox{0.9\columnwidth}{!}{%
    \begin{tabular}{c|cc|c}
        \toprule
        {\textbf{Method}} &
         \textbf{Google Robot} & \textbf{Widow X} &
        {\textbf{Avg.}} \\

        \midrule

        $\pi_0$
        & 56.4
        & 73.4
        & 63.9 \\

        $\pi_0$ + Masked Object
        & 6.0 {\down{50.4}}
        & 0.0 {\down{73.4}}
        & 3.0 {\down{61.9}} \\

        $\pi_0$ + Inpainting
        & 5.0 {\down{51.4}}
        & 0.0 {\down{73.4}}
        & 2.5 {\down{62.4}} \\

        \midrule

        GR00T N1.6
        & 45.6
        & 54.4
        & 49.5 \\

        GR00T N1.6 + Masked Object
        & 2.8 {\down{42.8}}
        & 0.0 {\down{54.4}}
        & 1.4 {\down{48.6}} \\

        GR00T N1.6 + Inpainting
        & 2.0 {\down{43.6}}
        & 0.0 {\down{54.4}}
        & 1.0 {\down{49.0}} \\

        \bottomrule
    \end{tabular}%
    }

\end{table}

\begin{table}[t]
    \centering
    \setlength{\tabcolsep}{3pt}
    \caption{\small{\textbf{Performance of $\pi_0$ and GR00T N1.6 under Object Masking and Random Masking.} Success rate (\%).}}
    \label{tab:random_masking}

    \resizebox{\columnwidth}{!}{%
    \begin{tabular}{c|c|cc|c}
        \toprule
        {\textbf{Method}} &
        {\textbf{Setting}} &
        \textbf{Google Robot} &
        \textbf{WidowX} &
        \textbf{Avg.} \\
        \midrule

        \multirow{3}{*}{$\pi_0$}
        & Baseline
        & 56.4
        & 73.4
        & 63.9 \\

        & + Object Masking
        & 6.0
        & 0.0
        & 3.0 \\

        & + Random Masking
        & 56.0
        & 72.0
        & 64.0 \\

        \midrule

        \multirow{3}{*}{GR00T N1.6}
        & Baseline
        & 45.6
        & 54.4
        & 50.0 \\

        & + Object Masking
        & 2.8
        & 0.0
        & 1.4 \\

        & + Random Masking
        & 46.0
        & 54.0
        & 50.0 \\

        \bottomrule
    \end{tabular}%
    }

\end{table}

We analyze $\pi_0$ and GR00T under masked observations (Table~\ref{tab:base_masking_image}). Masking task-relevant objects (Fig.~\ref{fig:visual_masking}) causes a \textit{catastrophic performance collapse}: $\pi_0$ drops by 61.9\% with Masked Object and 62.4\% with Inpainting, while GR00T drops by 48.6\% and 49.0\%, respectively. This degradation reveals reliance on ungrounded visual shortcuts, which we exploit to construct the negative action set for Stage 1.

To rule out masking artifacts, we additionally mask \textit{random task-irrelevant} regions (Table~\ref{tab:random_masking}). These masks have negligible impact, confirming that the performance collapse is primarily caused by removing task-relevant objects, which forces the models to rely more heavily on background cues. This supports the use of object masking to construct the negative set containing spurious action candidates.

\subsection{{\textbf{RQ3. Does the grounding mechanism actually select less-spurious actions?}}}
To investigate whether grounding actually selects actions with lower reliance on spurious visual features, we use the Nuisance Mass Ratio (NMR) from VLA-Explain \cite{zhang2026embodiedinterpretabilitylinkingcausal}, which measures the proportion of prediction-relevant tokens corresponding to irrelevant regions. Higher NMR indicates greater reliance on spurious features. We compare three variants: (1) \texttt{Base}, the original policy; (2) \texttt{Best Grounding}, selecting the candidate with the highest grounding score; and (3) \texttt{Worst Grounding}, selecting the lowest-scoring candidate.

As shown in Table~\ref{tab:nmr_grounding}, NMR consistently follows the ordering $\textit{Worst Grounding} > \textit{Base} > \textit{Best Grounding}$ across all four tasks. Best Grounding reduces the average NMR from $70.1\%$ to $68\%$, while Worst Grounding increases it to $75\%$. This provides empirical evidence that the grounding criterion preferentially selects less-spurious actions, as shown in Fig.~\ref{fig:vis_vla_explain}, and further supports our hypothesis that the negative set captures actions more strongly influenced by spurious visual features.

\begin{table}[H]
    \centering
    \setlength{\tabcolsep}{2.5pt}
    \caption{\small{\textbf{NMR scores of GR00T-N1.6 under different grounding-based action selections}. Lower NMR indicates less reliance on spurious visual features. The ordering $\textit{Worst Grounding} > \textit{Base} > \textit{Best Grounding}$ holds consistently, while increasing the candidate pool from $N=12$ to $N=36$ further reduces NMR.}}
    \label{tab:nmr_grounding}
    \vspace{-0.05in}
    \begin{tabular}{@{}l|ccccc@{}}
        \toprule
        Task &
        \makecell{+Worst\\Grounding\\($N=12$)} &
        Base &
        \makecell{+Best\\Grounding\\($N=12$)} &
        \makecell{+Best\\Grounding\\($N=24$)} &
        \makecell{+Best\\Grounding\\($N=36$)} \\
        \midrule
        Pick Coke Can & 0.740 & 0.668 & 0.663 & 0.657 & 0.656 \\
        Move Near     & 0.780 & 0.705 & 0.679 & 0.660 & 0.667 \\
        Carrot        & 0.770 & 0.746 & 0.712 & 0.705 & 0.698 \\
        Spoon         & 0.700 & 0.684 & 0.666 & 0.661 & 0.664 \\
        \midrule
        \textbf{Avg.} & \textbf{0.750} & \textbf{0.701} & \textbf{0.680} & \textbf{0.671} & \textbf{0.671} \\
        \bottomrule
    \end{tabular}
\vspace{-0.15in}
\end{table}

\vspace{0.05in}
\subsection{{\textbf{RQ4. How robust is SDN to different negative-set construction strategies?}}}
\label{exp:sdn_vs_pcd}

\begin{figure}
    \centering
    \includegraphics[width=0.95\linewidth]{images/vla_explain_v2.pdf}
    \caption{Visualization of VLA-Explain attribution maps. The best grounding concentrates model attention on task-relevant objects, while the worst grounding shows greater reliance on background regions.}
    \label{fig:vis_vla_explain}

\end{figure}

\begin{table}[!t]
    \centering
    \caption{\small{\textbf{Comparison of SDN and PCD under different perturbation strategies}, demonstrating that SDN remains effective even with simple perturbation strategies across different robot foundation models.}}
    \label{tab:sdn_pcd_comparison}
        \centering
        \setlength{\tabcolsep}{3pt}
        \begin{tabular}{l|ccccc}
            \toprule
            & $\pi_0$ & \multicolumn{2}{c}{\textbf{BBox-masking}} &
            \multicolumn{2}{c}{\textbf{Inpainting}} \\
            \cmidrule(lr){3-4} \cmidrule(lr){5-6}
            Task & Base. & SDN & PCD & SDN & PCD \\
            \midrule
            Close Drawer       & 75.7 & 87 & 78 & 82 & 75 \\
            Move Near          & 67.3 & 78 & 64 & 80 & 72.3 \\
            Open Drawer        & 38 & 50 & 52 & 46 & 56.3 \\
            Pick Coke Can      & 84 & 89 & 92 & 92 & 88 \\
            Apple Drawer       & 17 & 26 & 20 & 20 & 27.3 \\
            Carrot Plate       & 58 & 63 & 66 & 54 & 59.7 \\
            Eggplant Basket    & 86 & 92 & 72 & 96 & 87 \\
            Spoon Towel        & 80.7 & 90 & 92 & 88 & 84 \\
            Stack Cube         & 68.7 & 78 & 74 & 84 & 77 \\
            \midrule
            \textbf{Average} & \textbf{63.9} &
            \textbf{72.6} & 67.6 &
            \textbf{71.3} & 69.6 \\
            \bottomrule
            \toprule
            & GR00T N1.6 & \multicolumn{2}{c}{\textbf{BBox-masking}} &
            \multicolumn{2}{c}{\textbf{Inpainting}} \\
            \cmidrule(lr){3-4} \cmidrule(lr){5-6}
            Task & Base. & SDN & PCD & SDN & PCD \\
            \midrule
            Close Drawer       & 40 & 44 & 42 & 48 & 42 \\
            Move Near          & 76 & 90 & 78 & 90 & 80 \\
            Open Drawer        & 4 & 14 & 6 & 12 & 4 \\
            Pick Coke Can      & 90 & 94 & 94 & 96 & 92 \\
            Apple Drawer       & 18 & 20 & 14 & 20 & 12 \\
            Carrot Plate       & 60 & 60 & 60 & 62 & 62 \\
            Eggplant Basket    & 94 & 94 & 84 & 94 & 74 \\
            Spoon Towel        & 58.7 & 68 & 62 & 64 & 60 \\
            Stack Cube         & 4.7 & 6 & 6 & 6 & 6 \\
            \midrule
            \textbf{Average} & \textbf{49.5} &
            \textbf{54.4} & 49.6 &
            \textbf{54.7} & 48.0 \\
            \bottomrule
        \end{tabular}

\end{table}

We compare SDN with PCD using two negative-set construction strategies: bbox masking and inpainting-based perturbations. As shown in Table~\ref{tab:sdn_pcd_comparison}, SDN consistently achieves higher average success rates across both GR00T N1.6 and $\pi_0$. Specifically, SDN outperforms PCD by 4.8\% and 6.7\% for GR00T N1.6, and by 4.8\% and 1.7\% for $\pi_0$, under bbox masking and inpainting, respectively. These results show that SDN is robust to the negative-set construction strategy and foundation model, while achieving better performance without relying on the costly inpainting required by PCD for its best results.

\subsection{\textbf{Ablation Study}}
\paragraph{\texttt{Effect of Stage 1 vs. Stage 2.}}

Table~\ref{tab:realworld_results} evaluates SDN’s two-stage pipeline on real-world tasks involving rapid grasping, picking, and placing, where enforcing smoothness could potentially hurt task success. 
SDN-Grounding (Stage~1) improves performance by +13.33\% on average, while SDN-Smoothness (Stage~2) yields a larger gain of +16.67\%, showing that kinematic refinement improves reliability without compromising execution. 
Combining both stages achieves the best performance ($48.33\%$ avg.), demonstrating their complementary effects.

\vspace{0.05in}
\paragraph{\texttt{Robustness Analysis.}}
We evaluate SDN under varying sampling and action configurations. 
Figure~\ref{fig:ablation_efficiency}(a) shows that increasing $N$ improves success from 63.9\% to 72.6\% at $N=12$, after which performance largely saturates while latency increases substantially. 
Thus, $N=12$ provides a favorable trade-off between trajectory diversity and inference efficiency. 
Across different $k$, SDN remains effective, achieving the best performance of 74.6\% at $k=6$, indicating robustness of the contrastive grounding criterion to neighborhood size. 
Figure~\ref{fig:ablation_efficiency}(b) further demonstrates consistent gains across action horizons and both 7-DoF and 14-DoF action spaces, highlighting SDN's robustness to variations in action dimensionality and temporal horizon.

\begin{figure}[t]
    \centering
    \begin{subfigure}[t]{0.90\columnwidth}
        \centering
        \includegraphics[width=\columnwidth]{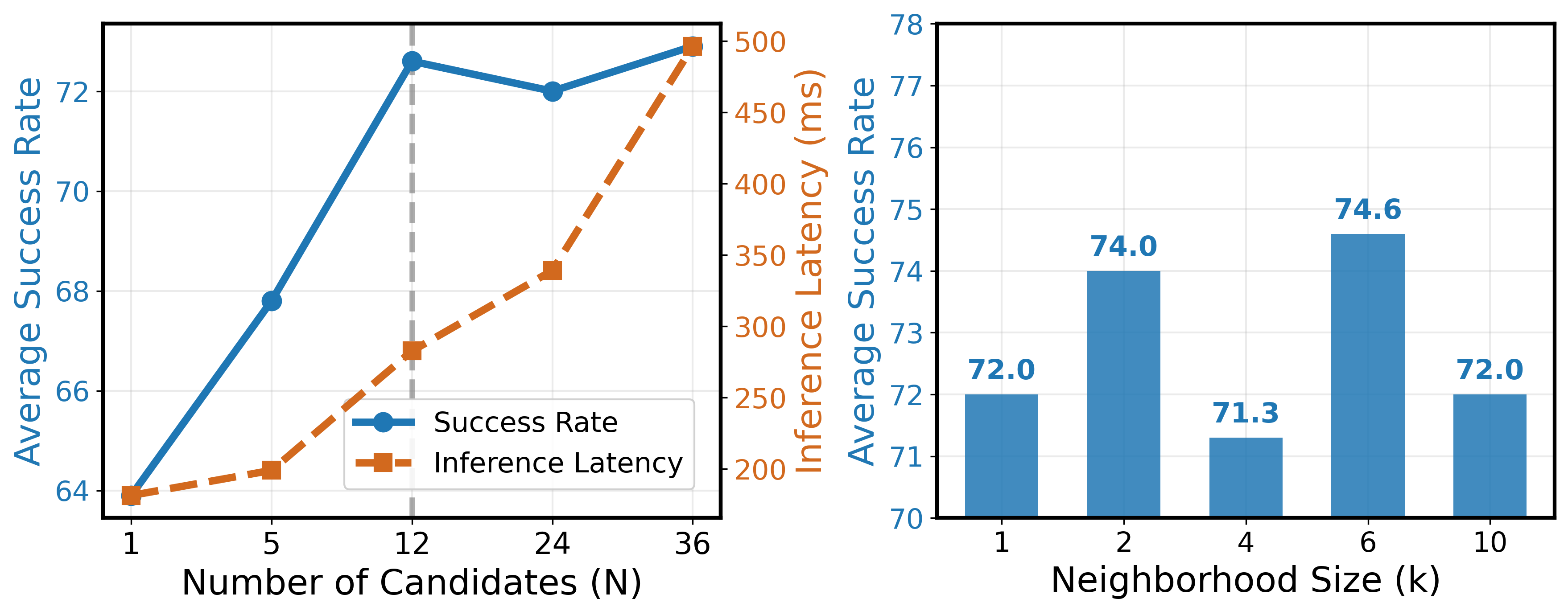}
        \caption{Effect of sample size $N$ and k-NN size $k$.}
        \label{fig:ablation_knn}
    \end{subfigure}

    \begin{subfigure}[t]{0.90\columnwidth}
        \centering
        \includegraphics[width=\columnwidth]{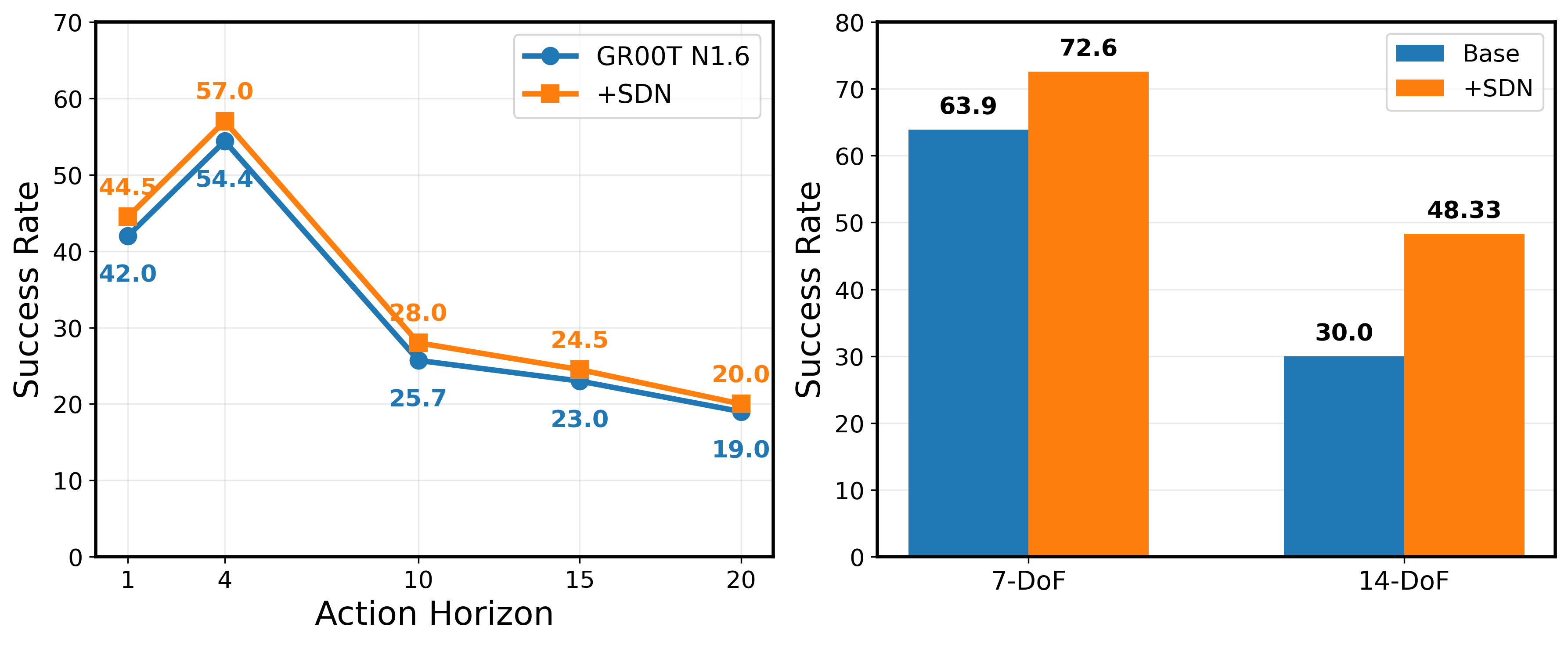}
        \caption{Robustness across action dimensions and horizons.}
        \label{fig:robustness_ah_dof}
    \end{subfigure}

    \caption{Robustness analysis of SDN.}
    \label{fig:ablation_efficiency}

\end{figure}

\vspace{0.1in}
\section{DISCUSSION \& LIMITATIONS}
We have shown that the initial noise in diffusion-based VLAs can serve as a controllable degree of freedom for mitigating spurious visual shortcuts and kinematically unstable motions, with consistent gains across simulation, real-world settings, and diverse backbones. Despite these gains, our perturbations may not cover all failure modes. SDN may suffer from \textit{corruption specificity}, as object-centric perturbations may miss spurious behaviors caused by dynamic lighting, texture shifts, or broader environmental distractions. Its reliance on fixed hyperparameters may also limit adaptivity to varying task complexity and uncertainty. Future work could explore adaptive candidate allocation and richer multimodal corruption modeling for improved robustness and generalization.

\bibliographystyle{IEEEtran}
\bibliography{IEEEabrv,references}

\section*{APPENDIX}

\subsection{Theoretical Analysis of SDN} \label{thm:grounding_consistency}
We provide theoretical insights into the behavior of the SDN grounding score as the number of samples $N$ increases.

\begin{theorem}
\textit{(Asymptotic consistency of SDN contrastive grounding).}\\
\label{thm:mixture}
Let $p_{\mathcal{G}}$ and $p_{\mathcal{B}}$ be absolutely continuous densities on a compact support $\Omega \subset \mathbb{R}^D$, bounded away from zero and infinity. Let
$\mathcal{G}_N = \{\actchunk_i\}_{i=1}^N$ and
$\mathcal{B}_N = \{\actchunk_j\}_{j=1}^N$
be i.i.d. samples from $p_{\mathcal{G}}$ and $p_{\mathcal{B}}$, respectively.

For any query $\actchunk \in \Omega$, define the average $k$-nearest-neighbor distances as
\begin{equation}
\small
d_{\mathcal{G}}^{(N)}(\actchunk)
=
\frac{1}{k}
\sum_{j \in \Omega_k^{\mathcal{G}}(\actchunk)}
\|\actchunk-\actchunk_j\|_2,
\end{equation}
and
\begin{equation}
\small
d_{\mathcal{B}}^{(N)}(\actchunk)
=
\frac{1}{k}
\sum_{j \in \Omega_k^{\mathcal{B}}(\actchunk)}
\|\actchunk-\actchunk_j\|_2.
\end{equation}

The SDN grounding score is defined as
\begin{equation}
\small
R_{\mathrm{ground}}^{(N)}(\actchunk)
=
d_{\mathcal{B}}^{(N)}(\actchunk)
-
d_{\mathcal{G}}^{(N)}(\actchunk).
\end{equation}

Assume $k \to \infty$, $N \to \infty$, and $k/N \to 0$. Then, for any fixed $\actchunk \in \Omega$,
\begin{equation}
\small
\begin{aligned}
R_{\mathrm{ground}}^{(N)}(\actchunk)
=
C_D
\left(\frac{k}{N}\right)^{1/D}
\Big(
&p_{\mathcal{B}}(\actchunk)^{-1/D} \\
&-
p_{\mathcal{G}}(\actchunk)^{-1/D}
\Big)
(1+o_P(1)),
\end{aligned}
\end{equation}
where $C_D = V_D^{-1/D}$ depends only on the dimension.

In particular, the ranking induced by
$R_{\mathrm{ground}}^{(N)}(\actchunk)$
is asymptotically consistent with the population score
\begin{equation}
\small
R(\actchunk)
=
p_{\mathcal{B}}(\actchunk)^{-1/D}
-
p_{\mathcal{G}}(\actchunk)^{-1/D}.
\end{equation}
\end{theorem}

\begin{proof}
From classical $k$-NN density estimation,
\begin{equation}
\small
\hat{p}_{\mathcal{S}_N}(\actchunk)
=
\frac{k}{N V_D \rho_k(\actchunk)^D}
\xrightarrow{P}
p_{\mathcal{S}}(\actchunk),
\end{equation}
where $\rho_k(\actchunk)$ is the $k$-NN radius.

This implies
\begin{equation}
\rho_k(\actchunk)
=
\left(
\frac{k}{N V_D p_{\mathcal{S}}(\actchunk)}
\right)^{1/D}
(1 + o_P(1)).
\end{equation}

Since the average $k$-NN distance is asymptotically equivalent to the radius,
\begin{equation}
d_{\mathcal{S}}^{(N)}(\actchunk)
=
\rho_k(\actchunk)(1 + o_P(1)),
\end{equation}
we obtain
\begin{equation}
d_{\mathcal{S}}^{(N)}(\actchunk)
=
C_D
\left(\frac{k}{N}\right)^{1/D}
p_{\mathcal{S}}(\actchunk)^{-1/D}
(1 + o_P(1)).
\end{equation}

Subtracting the expressions for $\mathcal{B}$ and $\mathcal{G}$ yields the result.
\end{proof}

\end{document}